\documentclass[conference]{IEEEtran}
\IEEEoverridecommandlockouts
\usepackage{cite}
\usepackage{amsmath,amssymb,amsfonts}
\usepackage{graphicx}
\usepackage{textcomp}
\usepackage{xcolor}
\usepackage{amsmath}
\usepackage{hyperref}
\usepackage{cleveref}
\usepackage{siunitx}
\usepackage{physics}
\usepackage{amsthm}
\usepackage{tikz}
\usepackage{algorithm}
\usepackage[noend]{algpseudocode}
\usepackage{subfig,graphicx}
\usepackage[font={footnotesize}]{caption}
\usepackage{float}
\usepackage{pifont}
\usepackage{booktabs}
\newcommand{\cmark}{\ding{51}}%
\newcommand{\xmark}{\ding{55}}%
\newcolumntype{P}[1]{>{\centering\arraybackslash}p{#1}}
\graphicspath{ {./img/} }
\def\BibTeX{{\rm B\kern-.05em{\sc i\kern-.025em b}\kern-.08em
    T\kern-.1667em\lower.7ex\hbox{E}\kern-.125emX}}
    
\newtheorem{theorem}{Theorem}[section]

\newtheorem{definition}{Definition}[section]

\crefname{figure}{Fig.}{Figs.}

\long\def\edit#1{\textcolor{black}{#1}}
    
\begin{document}

\title{Space-Time Conflict Spheres for Constrained Multi-Agent Motion Planning
\thanks{\textsuperscript{1} Anirudh Chari is with the Illinois Mathematics and Science Academy (\texttt{achari@imsa.edu})}
\thanks{\textsuperscript{2} Rui Chen and Changliu Liu are with the Robotics Institute, Carnegie Mellon University (\texttt{\{ruic3, cliu6\}@andrew.cmu.edu})}
}

\author{\IEEEauthorblockN{Anirudh Chari\textsuperscript{1}, Rui Chen\textsuperscript{2}, and Changliu Liu\textsuperscript{2}}
}

\maketitle

\begin{abstract}
Multi-agent motion planning (MAMP) is a critical challenge in applications such as connected autonomous vehicles and multi-robot systems. 
In this paper, we propose a space-time conflict resolution approach for MAMP. 
We formulate the problem using a novel, flexible sphere-based discretization for trajectories. 
Our approach leverages a depth-first conflict search strategy to provide the scalability of decoupled approaches while maintaining the computational guarantees of coupled approaches. 
We compose procedures for evading discretization error and adhering to kinematic constraints in generated solutions.
Theoretically, we prove the continuous-time feasibility and formulation-space completeness of our algorithm. 
Experimentally, we demonstrate that our algorithm matches the performance of the current state of the art with respect to both runtime and solution quality, while expanding upon the abilities of current work through accommodation for both static and dynamic obstacles.
We evaluate our algorithm in various unsignalized traffic intersection scenarios using CARLA, an open-source vehicle simulator.
Results show significant success rate improvement in spatially constrained settings, involving both connected and non-connected vehicles.
Furthermore, we maintain a reasonable suboptimality ratio that scales well among increasingly complex scenarios.
\end{abstract}


\section{Introduction}
Connected autonomous vehicles (CAVs) are self-driving vehicles that communicate with infrastructure and other vehicles. 
Vehicle-to-vehicle communication enables coordination among CAVs, which will greatly improve both the safety of road participants \cite{1,2} and the efficiency of traffic flow \cite{3}. 
We are particularly interested in the CAV coordination at traffic intersections \cite{4}, which are the site of a majority of road accidents due to human error \cite{5}.
The problem can be formulated as multi-agent motion planning (MAMP), which plans and coordinates trajectories among a group of agents such that each agent can travel from its start location to its goal without collisions with other agents or with the environment.
MAMP is also useful in surveillance, search-and-rescue, warehouse, and assembly robot groups. 
See \cite{6} for a thorough review.

MAMP is a generalization of the multi-agent path-finding problem (MAPF), where time is discretized into timesteps and agents move along the edges of a discrete graph. Finding an optimal solution to MAPF is NP-hard \cite{11}, hence optimal MAMP is also computationally intractable. 
There are generally two approaches to MAPF: coupled methods and decoupled methods. 
Coupled methods are often also referred to in literature as centralized, and decoupled methods are often referred to as decentralized or distributed.
Coupled methods search for solutions within a configuration space containing all agents, which enables guarantees of optimality and completeness. 
Despite recent advances in efficiency \cite{8,9,10}, coupled methods are unable to escape the exponential time complexity that comes with the high dimensionality of the configuration space.
On the other hand, decoupled methods consider agent paths individually before combining paths through conflict resolution strategies, which enables faster processing and better scalability, but with the drawback of difficulty in guaranteeing completeness and solution quality. 
In pursuit of computational tractability, we are motivated to further explore decoupled approaches.

Within decoupled MAMP approaches, there exist two primary conflict resolution strategies: temporal approaches and path prioritization. 
Temporal conflict resolution involves manipulating agent velocity profiles along respective paths, causing agents to pass through the conflict zone at different moments \cite{12,13,14}. 
In path prioritization, each agent is assigned a priority, and agents plan their paths sequentially in order of priority, with lower priority agents treating higher priority agents as dynamic obstacles \cite{18,19,20}.
Temporal methods are inherently suboptimal due to the omission of spatial trajectory manipulation. 
Path prioritization methods face an inevitable bottleneck due to the requirement of sequential processing.
\edit{Thus, we are motivated to address these drawbacks by pursuing a spatiotemporal conflict resolution strategy.}

Due to the applicability to CAVs and other domains, we also desire continuous-time feasibility and accommodation for dynamic obstacles as critical properties of a MAMP algorithm.
\edit{In the CAV domain, the former property ensures that generated solutions are safe during operation, and the latter property enables planning in commonplace scenarios involving pedestrians and non-connected vehicles.}
Recent work in decoupled MAPF has excelled in providing computational guarantees such as completeness \cite{21,24}. 
However, it is difficult to generalize MAPF solutions to MAMP problems, as discretization error may give rise to new conflict in continuous-time, and these approaches do not consider agent kinematic constraints.
Thus, we reach a dilemma: discretization of MAMP allows for easier formulation of robust and efficient algorithms, but at the expense of continuous-time feasibility, and consequently applicability to real-world settings.
On the other hand, attacking continuous MAMP directly risks computational intractability and complicates trajectory formulation, which may in turn restrict solution flexibility.
Some work attempts to adapt discrete solutions to the continuous problem \cite{26,35}, but no formal proof of feasibility in continuous time is present, and thus much of this work is unsuitable for real-world implementation.
\edit{While the authors of \cite{27} provide a proof of continuous-time feasibility for their discretization strategy, dynamic obstacles are ignored, which again hinders applicability.}
\edit{Current work in spatiotemporal conflict resolution \cite{mav} has neither of the aforementioned traits.}
To the best of the authors' knowledge, we are still missing literature regarding \textbf{decoupled, discretized MAMP algorithms with continuous-time feasibility and accommodation for dynamic obstacles}.

To fill this gap, we propose a novel decoupled approach to MAMP called the \underline{\textbf{S}}pace-\underline{\textbf{T}}ime \underline{\textbf{C}}onflict \underline{\textbf{S}}pheres (STCS) algorithm.
STCS utilizes a sphere-based trajectory discretization to manipulate paths both spatially and temporally during conflict resolution.
We theoretically prove STCS's continuous-time feasibility and formulation-space completeness.
We experimentally demonstrate that the algorithm exhibits comparable performance to the current state of the art while offering a greater range of problem settings and more consistency in finding solutions, namely within environments that are spatially constrained and contain dynamic obstacles.
The rest of the paper is organized as follows.
\Cref{probform} formulates the problem of MAMP. 
\Cref{stcs} discusses the STCS algorithm. 
\Cref{theoretical} discusses the theoretical properties of STCS. 
\Cref{results} presents experimental results in simulation. 
Finally, \Cref{conclusion} concludes our study and presents future directions.

\section{Problem Formulation}\label{probform}
We are given a set $\mathcal{A}$ containing $N$ \edit{communicating} agents of radius $r$ in a continuous-time, two-dimensional space, where each agent $\alpha_i \in \mathcal{A}$ is defined by its starting location $q^s_i$, goal location $q^g_i$, acceleration bound $a^{max}_i$, priority $\phi_i$, and path rigidity $\gamma_i$. 
We are also given a workspace $\mathcal{W}$ containing $M$ obstacles, where each static obstacle $o_i^s \in \mathcal{W}$ has a known location, and each dynamic obstacle $o_i^d \in \mathcal{W}$ has a known trajectory.
\edit{We assume noiseless agent intention communication and obstacle motion prediction.}

We use $\mathcal{P}$ to denote the set of agent and obstacle paths.
Each path $\pi_i \in \mathcal{P}$ can be represented spatiotemporally as a capsule, which we define as a curve inflated with radius $\lambda r$ for some $\lambda > 1$. 
The capsule representation is analogous to a chain of infinitely many spheres with radius $\lambda r$. 

We introduce a novel, flexible discretization of this representation to a finite chain of spheres, and introduce a requirement that adjacent spheres within a path must be no further than tangential to each other.
By \Cref{feasible-theorem}, this discrete representation maintains continuous-time feasibility (i.e. no discretization error) when $r$ is scaled by $\lambda^* = \frac{1}{\sqrt{3} - 1} \approx 1.366$, meaning paths can be interpolated safely for agent motion control.
For succinct representation, we can also introduce a requirement that alternating spheres within a path must be further than tangential to each other, ensuring that a path is always built from the minimum number of spheres.
\edit{In practice, we can further scale $r$ to increase robustness to error in communication and prediction.}

Using this representation, a path $\pi_i$ can be defined as a sequence of spatiotemporal spheres $\{s_{i,1}, ..., s_{i,n}\}$ as waypoints, where $s_{i,k}$ is the $k$-th waypoint in $\pi_i$.
A waypoint $s_{i,k}$ has components $(x_{i,k}, y_{i,k}, t_{i,k})$, where $(x_{i,k}, y_{i,k})$ is a physical location in the environment and $t_{i,k}$ is the time at which the location is occupied.
Each $s_{i,k}$ also has a corresponding velocity vector $\vec{\omega}_{i,k}$, which is computed based on conflict resolution conditions and kinematic constraints.

All paths in $\mathcal{P}$ can exist synchronously within a central \textit{space-time grid} (STG), which we define as a subspace of $\mathbb{R}^{3}$ with basis $\{\hat{x}, \hat{y}, \hat{t}\}$.
We assume that each agent $\alpha_i$ uses some single-agent motion planner to sequentially upload waypoints to its path $\pi_i$ in the STG.
In practice, either some centralized infrastructure can manage the STG while agents upload and query data, or each agent can maintain its own copy of the STG and send and receive broadcasts in a decentralized manner.
Among all paths in $\mathcal{P}$, an intersection between a pair of spheres belonging to distinct paths implies conflict, which must be resolved through the manipulation of paths in the STG.
We use $s_{i,k}^r$ and $s_{i,k}^*$ to denote the initial (reference) and final (optimal) locations of a sphere $s_{i,k}$, respectively.
Thus, we leverage our sphere-based discretization strategy to formulate conflict resolution as the optimization of $s_{i, k}^* \forall (i,k)$ as follows.\vspace{-10pt}
\begin{subequations}\label{eq:formulation}
\begin{align}
    \begin{split}
     \operatorname*{argmin}_{s_{i, k} \forall (i,k)} & \quad \displaystyle\sum_{i=1}^{n} \phi_i \norm{s_{i, k} - s_{i, k}^r}_2 \label{eq:obj}
    \end{split}
    \\
    \text{s.t.} & \quad \norm{s_{i,k} - s_{j,l}}_2 \geq 2r, \forall (s_{i,k}, s_{j,l}) \in \mathcal{P}, i \neq j \label{eq:con1} \\
                & \quad \norm{s_{i,k+1} - s_{i,k}}_2 \leq 2r, \forall (s_{i,k}, s_{i,k+1}) \in \mathcal{P} \label{eq:con2} \\
                & \quad t_{i,k+1} - t_{i,k} \leq (\delta t)_{i,k}, \forall (s_{i,k}, s_{i,k+1}) \in \mathcal{P} \label{eq:con3}
\end{align}
\end{subequations}
\Cref{eq:obj} minimizes the prioritized total displacement of spheres from their ``optimal'' original state in their respective path. 
\Cref{eq:con1} enforces that no pair of spheres from different paths can intersect. 
\Cref{eq:con2} enforces that consecutive spheres within a path must intersect. 
\Cref{eq:con3} enforces that time intervals between consecutive spheres within a path must be compatible with kinematic constraints; the computation of $(\delta t)_{i,k}$ is given in \labelcref{eq:deltat}.

Solving this optimization problem to resolve conflicts following the convergence of all agent paths to their respective goals would yield an optimal solution to MAMP. 
However, this leaves the task of solving a non-convex and potentially large optimization, which risks computational intractability. 
Instead, conflict resolution can be applied following each agent waypoint upload.
Employing this method as a heuristic, as we will observe, enables fast convergence to feasible solutions that are suboptimal within a reasonable bound.

\begin{figure*}[t]
    \centering
    \subfloat[Conflict detection]{\label{fig:ex_a}\includegraphics[width=.18\linewidth]{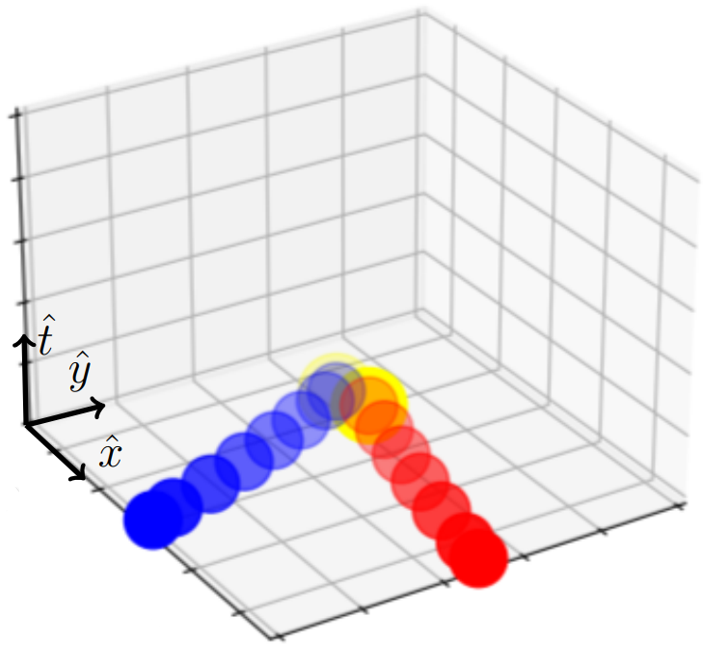}}
    \hspace{1em}
    \subfloat[DV Computation]{\label{fig:ex_b}\includegraphics[width=.18\linewidth] {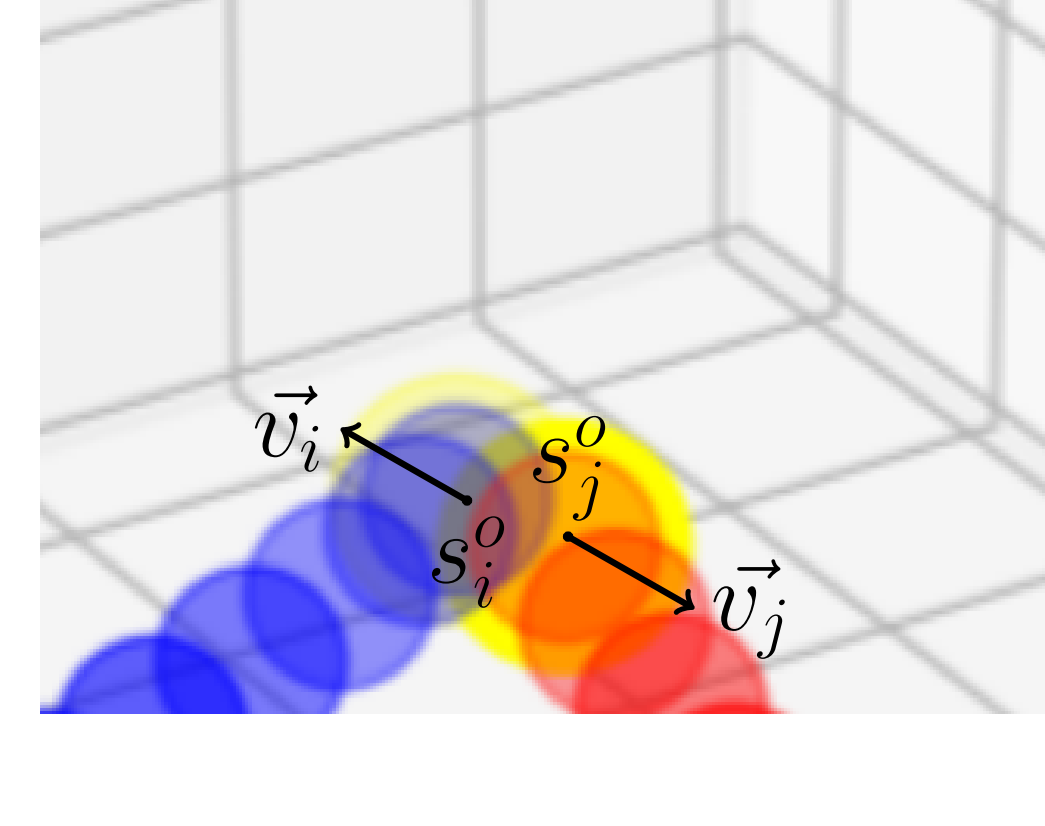}}
    \hspace{1em}
    \subfloat[Path Shift]
    {\label{fig:ex_c}\includegraphics[width=.2\linewidth]{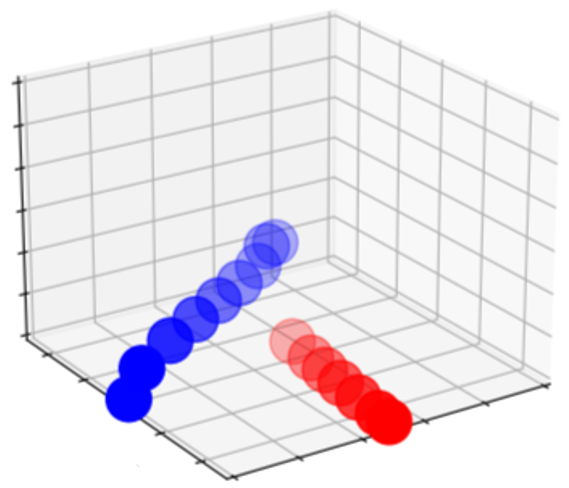}}
    \hspace{1em}
    \subfloat[Final solution] {\label{fig:ex_d}\includegraphics[width=.2\linewidth]{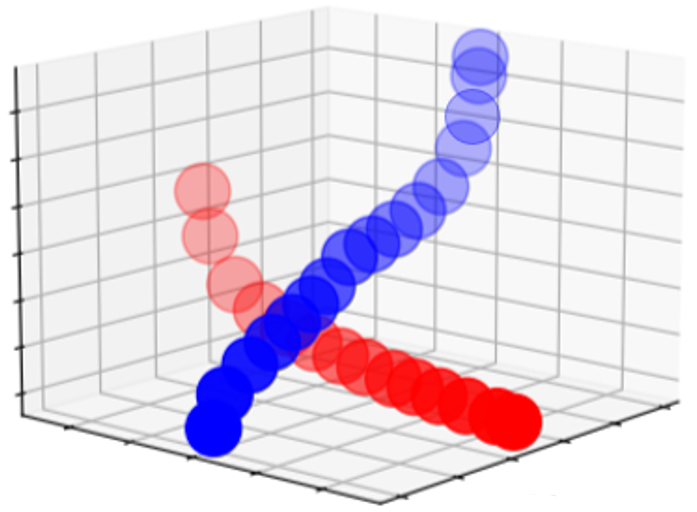}}
    \caption{\edit{STG representation of a} simple scenario involving two agents initially positioned perpendicular to each other, each with a goal directly across the field. The conflict is highlighted in yellow. Because the red path is given higher priority, the conflict search finds that shifting the blue path yields the best solution.}
    \label{fig:example}
    \vspace{-14pt}
\end{figure*}

\section{Space-Time Conflict Spheres}\label{stcs}

\subsection{Overview}
The objective of STCS is to resolve conflicts upon formation while attempting to minimize the total displacement of spheres during this process.
Simultaneously, the algorithm must ensure that connectivity \labelcref{eq:con2} and compliance with kinematic constraints \labelcref{eq:con3} is maintained within each path.
Resolving one intersection may lead to the formation of many others, making this a difficult problem.

We begin by approaching the sub-problem of computing the minimum displacement required to resolve an intersection between a single pair of spheres, and we solve this by introducing the idea of \textit{displacement vectors} (DVs).
Then, we move to considering the effect of a sphere's translation on its individual path, namely through a \textit{path shift}, which first deforms a path around its displaced sphere then applies a smoothing operation to maintain connectivity and adherence to temporal and kinematic constraints.
Finally, we employ the two above concepts and formulate an efficient search procedure for collecting complete solutions to the current conflict using a depth-first paradigm; we refer to this process as \textit{conflict search}.
A visual overview of the process is given in \Cref{fig:example}.
Each of the following three subsections details one of the aforementioned aspects of STCS.
\vspace{-3pt}

\subsection{Displacement Vectors}
We define an \textit{outstanding} sphere as one in the STG with the potential to be involved in an intersection with another sphere of a different path.
We denote the set of all outstanding spheres in a solution as $\mathcal{S}$.
Each sphere $s^o_i \in \mathcal{S}$ has a single corresponding DV $\vec{v}_i \in \mathcal{V}$, where $\mathcal{V}$ is the set of all DVs.
We compute vectors for each of the outstanding spheres such that applying $\vec{v}_i$ to $s^o_i$ yields an intersection-free STG.

The translation caused by DVs for a pair of intersecting spheres can be visualized as a repulsive force acting between two charged particles.
In general, the DV $\vec{v}_i$ of outstanding sphere $s^o_i$ to resolve intersection with sphere $s^o_j$ is given by 
$\vec{v}_i = \left(2r/\norm{s^o_i - s^o_j}_2 - 1\right) (s^o_i - s^o_j)$.
The magnitude of $\vec{v}_i$ is the minimum displacement of $s^o_i$ necessary to resolve the intersection, assuming $s^o_j$ is stationary.
The direction of $\vec{v}_i$ is orthogonal to the plane of intersection between $s^o_i$ and $s^o_j$ (see \Cref{fig:dvspheres}).
In the case where two paths advance straight towards one another, we can introduce a small bias in the angle of each DV in the conflicting sphere pair to ensure the paths can navigate around each other.
Note that the first sphere in any path is immutable, and the final sphere in any path that has converged to its goal is immutable spatially.

\subsection{Path Shifts}
Path shifts are first simulated during the conflict search stage, then finally applied post-optimization. 
The displacement of any outstanding sphere along its DV will cause a shift within that sphere's path, centered around it. 
This can be intuitively visualized as a rubber rod deforming after being hit by a ball (see \Cref{fig:pathshift}). 
For a path $\pi_i$ that contains an outstanding sphere $s^o_j$ with DV $\vec{v}_j$, for each sphere $s_{i,k} \in \pi_i$, we define the coefficient $\mu_{i,k} = \exp(-\gamma_i(d_{i,k}/d_i^{max})^2),$
where $d_{i,k}$ is the distance between $s_{i,k}$ and the outstanding sphere $s^o_j$, $d_i^{max}$ is the maximum distance between any point along $\pi_i$ and $s^o_j$ (see \Cref{fig:shift_a}).
Furthermore, $\gamma_i$ is a positive constant assigning the path rigidity of $\pi_i$: larger $\gamma_i$ localizes the effects of the collision around the outstanding sphere, while smaller values resonate the effects throughout the path. 
Then, each $s_{i,k}$ is accordingly translated along a path shift vector $\vec{\psi}_{i,k} = \mu_{i,k} \vec{v}_j$.
By adapting basic kinematics equations and solving for time, the minimum timestep $(\delta t)_{i,k}$ required for an agent $\alpha_i$ to traverse between points $s_{i,k}$ and $s_{i,k+1}$ on its path is
\begin{equation}\label{eq:deltat}
    (\delta t)_{i,k} = \frac{-\norm{\vec{\omega}^{\sigma}_{i,k}}_2 + \sqrt{\norm{\vec{\omega}^{\sigma}_{i,k}}_2^2 + 2a^{max}_i \norm{\vec{\sigma}_{i,k}}_2}}{a^{max}_i}
\end{equation}
where $\vec{\sigma}_{i,k}$ denotes the projection of $(s_{i,k+1} - s_{i,k})$ onto $\{\hat{x}, \hat{y}\}$, i.e. the spatial displacement between $s_{i,k}$ and $s_{i,k+1}$, $\vec{\omega}^{\sigma}_{i,k}$ is the projection of agent $\alpha_i$'s velocity vector $\vec{\omega}_{i,k}$ at $s_{i,k}$ onto $\vec{\sigma}_{i,k}$, i.e. the velocity of $\alpha_i$ along $\pi_i$ at point $s_{i,k}$, and $a^{max}_i$ is the agent's acceleration bound. 

After applying $\vec{\psi}_{i,k}$, the following path-smoothing operation is executed by iterating forward through the current agent path $\pi_i$, which computes a velocity profile and removes all kinematic constraint violations by $\vec{\omega}_{i,k} := \vec{\omega}_{i,k-1} + a^{max}_i (\delta t)_{i,k} \hat{\sigma}_{i,k}$ and $t_{i,k} := \max(t_{i,k}, t_{i,k-1} + (\delta t)_{i,k})$.
This smoothing strategy pushes the trajectory to its kinematic limits by maximizing velocity while ensuring agreement with kinematic constraints.

For scenarios in which agent paths are spatially constrained along a general direction, e.g. lane markings at a traffic intersection, it is simple to introduce a requirement that makes certain spheres along the trajectory immutable spatially.

\begin{figure}[t]
    \hspace{0.1cm}
    \begin{minipage}[b]{0.4\linewidth}
        \centering
        \includegraphics[width=\textwidth]{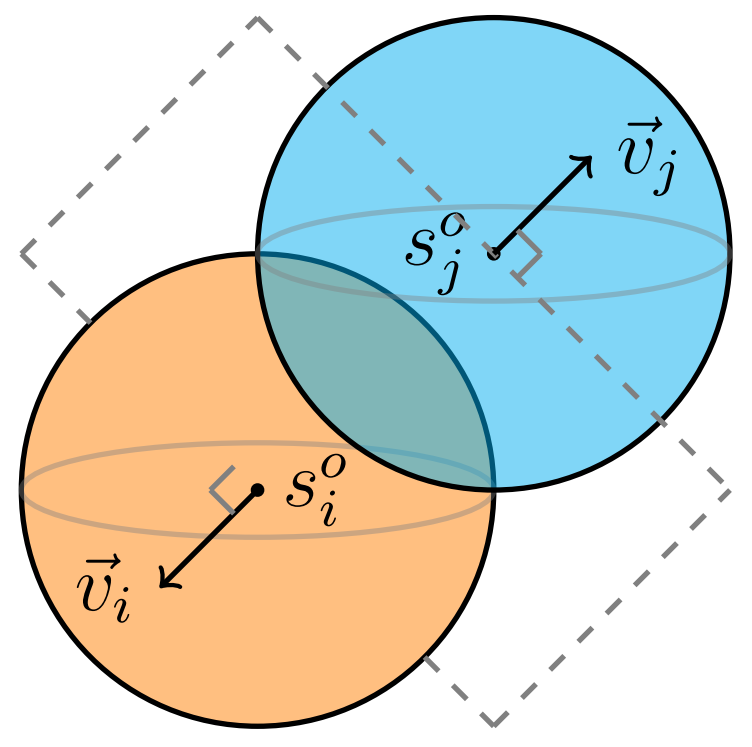}
        \caption{Computed DVs for a pair of intersecting spheres.}
        \label{fig:dvspheres}
    \end{minipage}\vspace{-12pt}
    \hspace{1cm}
    \begin{minipage}[b]{0.4\linewidth}
        \centering
        \subfloat[Pre-shift]{\label{fig:shift_a}\includegraphics[width=\textwidth]{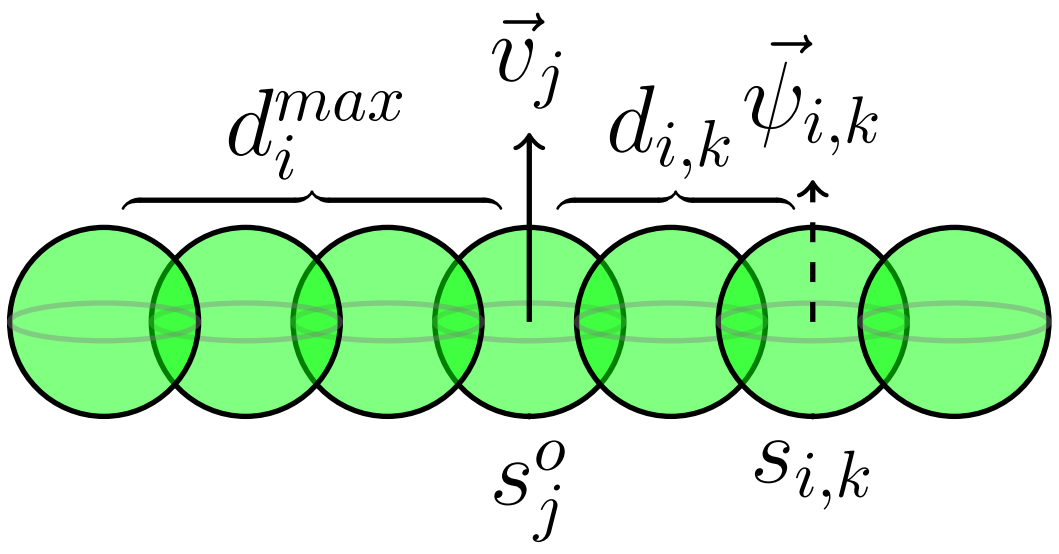}}

        \subfloat[Post-shift] {\label{fig:shift_b}\includegraphics[width=\textwidth]{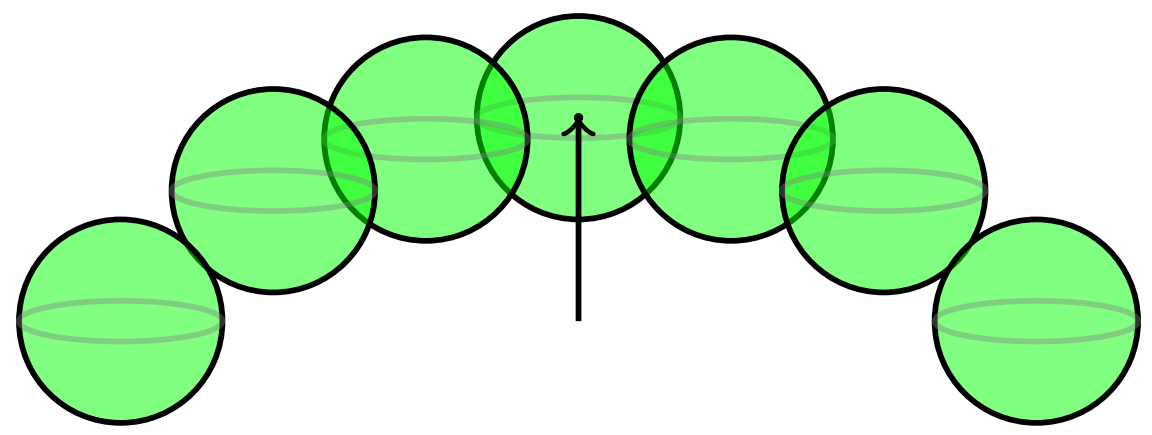}}
        \caption{A path shift centered around $s^o_j$ along $\vec{v}_j$.}
        \label{fig:pathshift}
    \end{minipage}
\end{figure}
\vspace{-6pt}

\subsection{Conflict Search}
We can utilize a three-dimensional range querying data structure (e.g. $k$-d tree) to efficiently query pairs of intersecting spheres within the STG during each iteration of conflict resolution.
To determine the set of outstanding spheres $\mathcal{S}$ and their respective DVs $\mathcal{V}$, we employ a recursive, depth-first conflict search. 
We define \textit{calling} a sphere as translating the sphere across some specified DV within some current STG state $\mathcal{T}_{cur}$, applying a path shift, querying further intersections, computing the DV of each sphere involved in an intersection, and finally calling each of these involved spheres with their respective DVs and the new STG state $\mathcal{T}_{new}$. 
There are two base cases for this recursion when a sphere is called: if the sphere has already been visited in the current recursion sequence, it returns false (infeasible), and if no more intersections arise following the sphere's translation and path shift, it returns true (feasible). 

By recursively decomposing each intersection into two cases (e.g. $s_1$ moves vs. $s_2$ moves), a set of solutions $\mathcal{L}$ is accumulated, where each solution $\mathcal{L}_i$ is an object containing a particular sequence $\{\mathcal{S}_i, \mathcal{V}_i\}$ that resolves all conflict.
Multiple solutions are obtainable since each intersection can be decomposed into two cases (e.g. $s_1$ moves vs. $s_2$ moves).
In the case of an agent-obstacle sphere intersection, only the agent sphere can be called.
In the case of a chain reaction of intersections (e.g. $s_1$ intersects $s_3$ after resolving intersection with $s_2$), the requirement that visited spheres cannot be called implies that only one new sphere will be called.
\edit{
Once all solutions have been collected, the best can be selected by minimizing the objective function $\sum_{i=1}^{n} \phi_i \norm{\vec{v}_i}_2$, where $\phi_i$ is the priority value of the path that $s^o_i$ belongs to. 
Note that this function is a refinement of \labelcref{eq:obj} that only allows the manipulation of outstanding spheres, and restricts the movement of these spheres to the magnitude and direction of their respective DVs.
}
This formulation enables parallelized path prioritization since trajectories can be planned simultaneously while still implicitly favoring high-priority agents during conflict.
The above procedure is summarized in \Cref{alg:conflictsearch}.
We will show in \Cref{complete-theorem} that this algorithm is complete with respect to the formulation space $\mathcal{F}$ of the conflict (see \Cref{formulation-def}).

\begin{algorithm}[t]
\caption{Conflict search}
\label{alg:conflictsearch}
    \begin{algorithmic}[1]
    \Statex
        \Function{Resolve}{$\mathcal{T}_{cur}, s_{cur}, \vec{v}_{cur}, vis$}
            \If{$vis[s_{cur}]$}
                \State \Return $[\hspace{1mm}]$ \Comment{Infeasible, already visited}
            \EndIf
            \State $vis[s_{cur}] \gets \text{True}$, $\mathcal{L}_{cur} \gets [\hspace{1mm}]$ 
            \State $\mathcal{T}_{new} \gets$ \Call{PathShift}{$\mathcal{T}_{cur}, s_{cur}, \vec{v}_{cur}$}
            \State $query \gets$ \Call{QueryPairs}{$\mathcal{T}_{new}$}
            \State $feasible \gets \text{False}$
            \ForAll{$(s_{i,k}, s_{j,l})$ in $query$}
                \State $\vec{v}_{i,k} \gets$ \Call{ComputeDV}{$\mathcal{T}_{new}[s_{i,k}], \mathcal{T}_{new}[s_{j,l}]$}
                \State $\vec{v}_{j,l} \gets$ \Call{ComputeDV}{$\mathcal{T}_{new}[s_{j,l}], \mathcal{T}_{new}[s_{i,k}]$}
                \State $\mathcal{L}_{i,k} \gets$ \Call{Resolve}{$T_{new}, s_{i,k}, \vec{v}_{i,k}, vis$}
                \State $\mathcal{L}_{j,l} \gets$ \Call{Resolve}{$T_{new}, s_{j,l}, \vec{v}_{j,l}, vis$}
                \State $\mathcal{L}_{cur} \mathrel{+}= \mathcal{L}_{i,k} + \mathcal{L}_{j,l}$
                \If{$(\mathcal{L}_{i,k} + \mathcal{L}_{j,l})$ not empty}
                    \State $feasible \gets \text{True}$
                \EndIf
            \EndFor
            \If{not $feasible$ and $query$ not empty}
                \State \Return $[\hspace{1mm}]$ \Comment{Infeasible, unresolved conflict}
            \EndIf
            \If{$\mathcal{L}_{cur}$ empty}
                \State \Call{Push}{$\mathcal{L}_{cur}, [\hspace{1mm}]$} \Comment{Feasible, end of solution}
            \EndIf
            \ForAll{$sol$ in $\mathcal{L}_{cur}$}
                \State \Call{Push}{$sol, (s_{cur}, \vec{v}_{cur})$} \Comment{Feasible, build solutions}
            \EndFor
        \State \Return $\mathcal{L}_{cur}$ \Comment{All solutions}
        \EndFunction
    \end{algorithmic}
\end{algorithm}

\section{Theoretical Properties}\label{theoretical}

\begin{definition}[feasibility]
    We refer to a MAMP solution as \textbf{feasible} if the computed path configuration is conflict-free in the continuous time domain.
\end{definition}

\begin{theorem}[continuous-time feasibility]\label{feasible-theorem} If there exists a solution to the discrete-time problem \labelcref{eq:formulation} when $r$ is scaled by $\lambda^* = \frac{1}{\sqrt{3} - 1}$, then the solution is feasible.
\end{theorem}

\begin{proof}
The chain-of-spheres path representation is discrete, and thus discretization error is inherent. 
The error occurs if an intersection exists in the capsule representation, but not in the chain-of-spheres representation. 
Suppose we have a sphere $s_{i,k}$ from path $\pi_i$, and two adjacent spheres $s_{j,l}$ and $s_{j,l+1}$ from a second path $p_j$. 
Here, discretization error would occur if $s_{i,k}$ intersects neither $s_{j,l}$ nor $s_{j,l+1}$, but it is still within the bounding capsule $c$ between $s_{j,l}$ and $s_{j,l+1}$. 
To resolve this violation, we can scale the radius of all spheres in the space by some constant $\lambda > 1$ during conflict resolution. 
By geometry, $s_{i,k}$ is exactly tangential to $c$ when $\lambda^* = \frac{1}{\sqrt{3} - 1}$. 
\end{proof}

\begin{definition}[formulation space]\label{formulation-def}
We define the \textbf{formulation space} $\mathcal{F}$ of a MAMP conflict as the set of all possible path configurations that can be reached from some initial state by executing some sequence $\mathcal{V}$ of DV translations, given the formulation of the DV computation and path shift operations.
\end{definition}

\begin{theorem}[formulation-space completeness]\label{complete-theorem}
 If there exists a solution to \labelcref{eq:formulation} that also exists in $\mathcal{F}$, then STCS will find and return a feasible solution.
\end{theorem}

\begin{proof}
By nature, the depth-first search performs a complete search of the solution space.
Generally, we can state that if $\vec{v}_i$ will yield a feasible solution following some additional sequence of translations $\{\vec{v}_{i+1}, \vec{v}_{i+2}, ..., \vec{v}_{n-1}, \vec{v}_{n}\}$, then $\vec{v}_{i-1}$ will also yield that same feasible solution, given $\vec{v}_i$ and the same additional sequence.
Because DV translations and subsequent path shifts are applied in the same order in which they are computed and path shifts are simulated during the conflict search stage, a solution in the form of a DV sequence must yield an intersection-free STG.
By \Cref{feasible-theorem}, an intersection-free STG implies a feasible MAMP solution.
\end{proof}

\section{Experimental Results}\label{results}
We first evaluate STCS under motion planning tasks, and then verify the planned trajectories in realistic traffic simulation.
We compare our algorithm to a baseline as well.

\subsection{Simulation Setup}
We simulate motion planning tasks where $N$ agents are given starting and goal locations and must cooperatively navigate through an environment with $M$ static and dynamic obstacles, for $N \in [2, 4]$ and $M \in [0, 5]$.
We implement STCS in Python to configure and solve each motion planning instance.
The planned trajectories from each scenario are then executed in a traffic intersection using CARLA, an open-source autonomous driving simulator \cite{30}.
During CARLA evaluation, we verify the feasibility of solutions through adherence to car dynamics for control and spatial constraint of agent motion to the dimensions of the intersection.
In all experiments, the field size is $20 \si{\meter} \times 20 \si{\meter}$, and for all agents, we assign the radius $r = 3.5 \si{\meter}$, the acceleration bound $a^{max}_i = \pm 3 \si[per-mode=symbol]{\meter\per\second\squared}$, and the path rigidity $\gamma_i = 10$. 
We assign the priority $\phi_i=100$ for Agent $1$ and $\phi_i=1$ for all other agents.
For evaluation, we configure various scenarios \edit{(see \Cref{fig:sim_setup})} in four categories: \textbf{obstacle-free} (F1-F3), \textbf{static obstacles} (S1-S3), \textbf{dynamic obstacles} (D1-D3) and \textbf{non-connected vehicles} (N1-N3).
The non-connected vehicle case is a special case of dynamic obstacles that is particularly applicable to CAVs.

\begin{figure}[t]
    \centering
    \includegraphics[width=\linewidth]{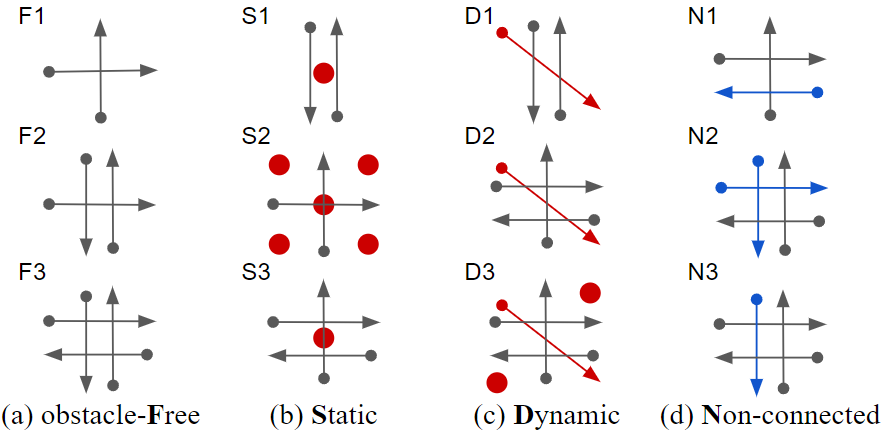}
    \caption{Simulated motion planning tasks. Agents are depicted in black arrows. Obstacles are depicted in red dots (static) or arrows (dynamic). Blue arrows represent non-connected vehicles.}
    \label{fig:sim_setup}
    \vspace{-15pt}
\end{figure}

\subsection{Baseline and Evaluation Metric}
To assess solution quality, we compute \textit{suboptimality ratios} for various metrics, which we define as the ratio between the observed value of the metric and a lower bound on its optimal value.
In particular, we measure total distance, as the sum of the distances traveled by all agents in the solution, and makespan, which is the time that the last agent reaches its goal.
We obtain a lower bound on the optimal total distance of a solution using the sum of L2 norms between each agent's starting point and its goal. 
We obtain a lower bound on the optimal makespan of a solution by solving for time using basic kinematics formulas, assuming that traveled distance is the L2 norm between an agent's starting point and its goal and that the agent has a constant acceleration $a^{max}_i$. 
We finally compute the \textit{overall} suboptimality ratio as the average ratio over all metrics.
The overall suboptimality ratio is at least $1$, the lower the better.
A ratio of $1$ is often unattainable in practice.
We compare the performance of STCS to that of S2M2, a MAMP algorithm proposed in \cite{25} which is incompatible with dynamic obstacles. 
We measure each algorithm's runtime in each scenario as the average runtime over $20$ trials.
All experiments were run on a desktop computer with an AMD Ryzen 5 2600X CPU and 16GB RAM.

\subsection{Motion Planning Evaluation}
We evaluate STCS in motion planning with respect to runtime and suboptimality ratio, compared to S2M2 in obstacle-free and static-obstacle scenarios.
\Cref{fig:runtime_f} indicates that STCS and S2M2 provide similar runtime in obstacle-free scenarios.
\Cref{fig:subopt_f} indicates that the algorithms also exhibit comparable solution quality in these scenarios, as implied by the overall suboptimality ratio metric.
Likewise, \Cref{fig:runtime_s} and \Cref{fig:subopt_s} extend these observations to scenarios involving static obstacles.
\Cref{tab:dynamic} and \Cref{tab:nonconnected} demonstrate that the performance of STCS with respect to both runtime and solution quality scales well to scenarios involving dynamic obstacles and non-connected vehicles.
Thus, STCS matches state-of-the-art performance with respect to both runtime and solution quality in obstacle-free and static-obstacle cases and extends this performance to settings involving dynamic obstacles.

\begin{figure}[t]
    \centering
    \begin{minipage}{.475\columnwidth}
        \includegraphics[width=\linewidth]{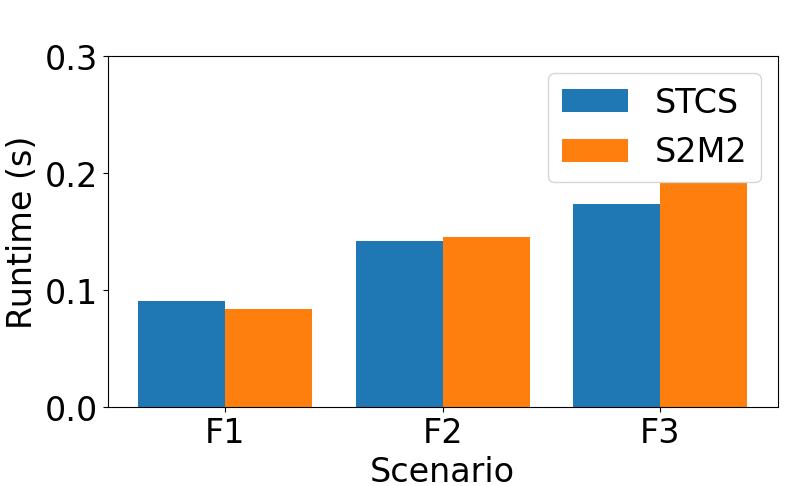}
        \caption{Avg runtime.}
        \label{fig:runtime_f}
    \end{minipage}\vspace{-14pt}\hfill
    \begin{minipage}{.475\columnwidth}
        \includegraphics[width=\linewidth]{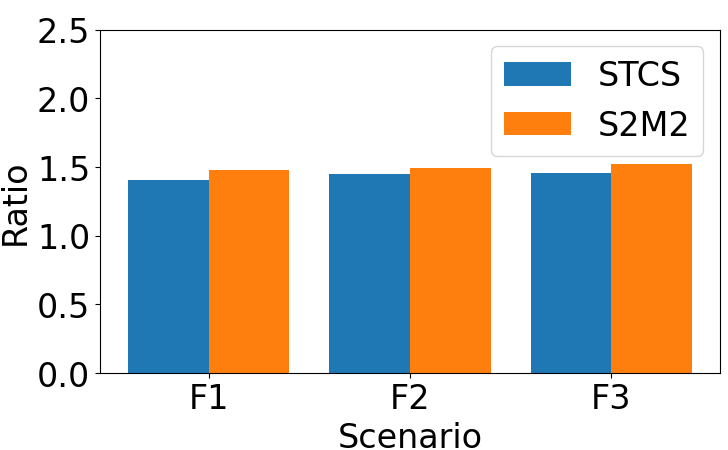}
        \caption{Overall suboptimality ratio.}
        \label{fig:subopt_f}
    \end{minipage}
\end{figure}

\begin{figure}[t]
    \centering
    \begin{minipage}{.475\columnwidth}
        \includegraphics[width=\linewidth]{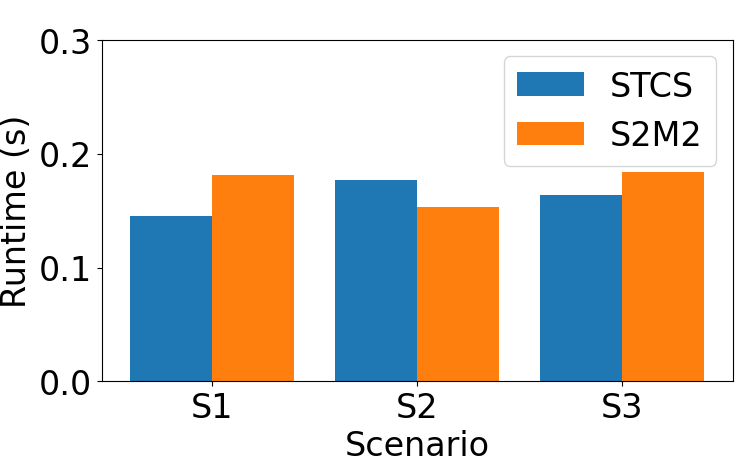}
        \caption{Avg runtime.}
        \label{fig:runtime_s}
    \end{minipage}\vspace{-10pt}\hfill
    \begin{minipage}{.475\columnwidth}
        \includegraphics[width=\linewidth]{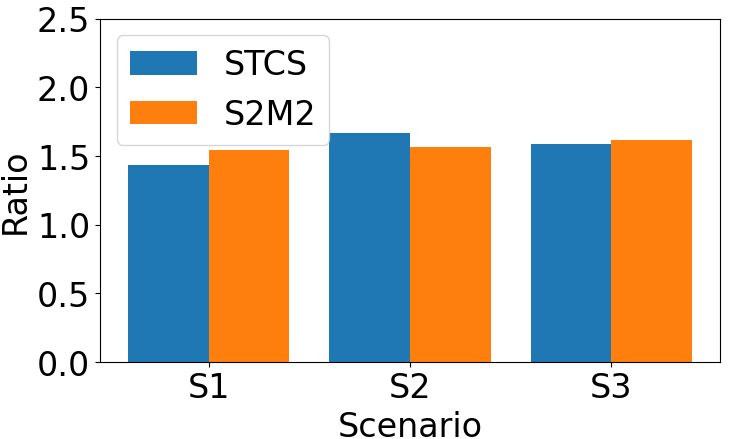}
        \caption{Overall suboptimality ratio.}
        \label{fig:subopt_s}
    \end{minipage}
    \vspace{-10pt}
\end{figure}

\begin{table}[t]
\begin{center}
    \begin{tabular}{p{3.25cm}P{1.25cm}P{1.25cm}P{1.25cm}}
        \toprule
        & \textbf{D1} & \textbf{D2} & \textbf{D3} \\
        \midrule
        \textbf{Runtime} (\si{\second}) & 0.231 & 0.246 & 0.212 \\
        \textbf{Distance Suboptimality} & 1.059 & 1.094 & 1.128 \\
        \textbf{Makespan Suboptimality} & 1.825 & 2.424 & 2.324 \\
        \textbf{Overall Suboptimality} & 1.442 & 1.759 & 1.726 \\
        \bottomrule
    \end{tabular}
    \caption{Evaluation of STCS in dynamic obstacle scenarios.}
    \label{tab:dynamic}
\end{center}
\vspace{-10pt}
\end{table}  

\begin{table}[t]
\begin{center}
    \begin{tabular}{p{3.25cm}P{1.25cm}P{1.25cm}P{1.25cm}}
        \toprule
        & \textbf{N1} & \textbf{N2} & \textbf{N3} \\
        \midrule
        \textbf{Runtime} (\si{\second}) & 0.155 & 0.224 & 0.237 \\
        \textbf{Distance Suboptimality} & 1.561 & 1.146 & 1.785 \\
        \textbf{Makespan Suboptimality} & 1.657 & 2.281 & 1.848 \\
        \textbf{Overall Suboptimality} & 1.609 & 1.713 & 1.817 \\
        \bottomrule
    \end{tabular}
    \caption{Evaluation of STCS in non-connected vehicle scenarios.}
    \label{tab:nonconnected}
\end{center}
\vspace{-17pt}
\end{table}  

\subsection{CARLA Evaluation}
We further evaluate STCS in traffic intersections using the CARLA autonomous vehicle simulator. 
We use a four-way, two-lane uncontrolled intersection setting.
We assign cars for each agent and non-connected vehicle, and bicycles and pedestrians for each static or dynamic obstacle.
Furthermore, we constrain vehicles the bounds of the intersection. 
In each scenario, we test if STCS and S2M2 produce a solution that can be feasibly executed in CARLA.
A full solution generated by STCS for Scenario \textbf{N2} is shown in \Cref{fig:carla_N2}.

As indicated by \Cref{tab:carla}, the MAMP solutions produced by STCS were feasible for all types of scenarios.
On the other hand, S2M2 occasionally failed given the spatial constraints, specifically in the more challenging \textbf{F3} and \textbf{S3} cases.
Thus, we experimentally demonstrate that STCS expands upon the current state of the art by solving complex traffic intersection scenarios involving pedestrians and non-connected vehicles, and by maintaining continuous-time feasibility and formulation-space completeness in all cases.
In particular, we show that the latter property enables solutions to be found more consistently in challenging scenarios with spatial constraints.

\begin{table}[t]
    \begin{minipage}{.525\columnwidth}
        \begin{center}
        \begin{tabular}{ P{1.2cm}P{1.2cm}P{1.2cm} }
            \toprule
            & \textbf{STCS} & \textbf{S2M2} \\
            \midrule
            \textbf{F1} & \cmark & \cmark \\
            \textbf{F2} & \cmark & \cmark \\
            \textbf{F3} & \cmark & \xmark \\
            \midrule
            \textbf{S1} & \cmark & \cmark \\
            \textbf{S2} & \cmark & \cmark \\
            \textbf{S3} & \cmark & \xmark \\
            \bottomrule
        \end{tabular}
        \end{center}
    \end{minipage}\hfill
    \begin{minipage}{.425\columnwidth}
        \begin{center}
        \begin{tabular}{ P{1.2cm}P{1.2cm} }
            \toprule
            & \textbf{STCS} \\
            \midrule
            \textbf{D1} & \cmark\\
            \textbf{D2} & \cmark\\
            \textbf{D3} & \cmark\\
            \midrule
            \textbf{N1} & \cmark\\
            \textbf{N2} & \cmark\\
            \textbf{N3} & \cmark\\
            \bottomrule
        \end{tabular}
        \end{center}
    \end{minipage}
\caption{Feasibility of solutions by STCS and S2M2 in CARLA.}
\label{tab:carla}
\vspace{-10pt}
\end{table}

\begin{figure}[t]
    \centering
    \includegraphics[width=0.7\linewidth]{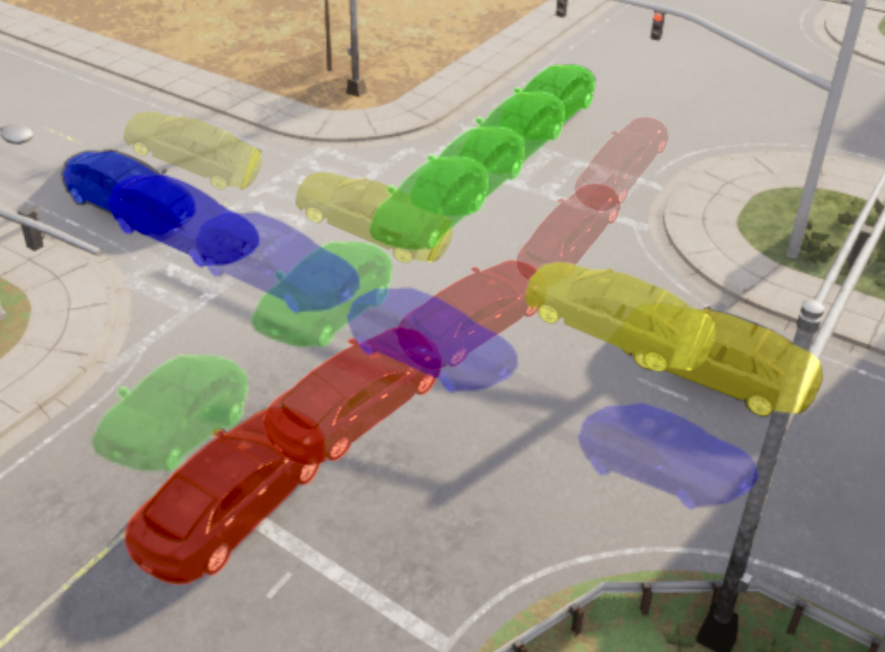}
    \caption{Trajectories generated by STCS for Scenario \textbf{N2} in CARLA, where the red and yellow belong to Agents 1 and 2, respectively, and the blue and green paths belong to non-connected vehicles.}
    \label{fig:carla_N2}
    \vspace{-12pt}
\end{figure}

\vspace{-5pt}
\section{Conclusion}\label{conclusion}
In this work we presented STCS, a novel discrete-time formulation and decoupled algorithm for multi-agent motion planning.
We theoretically proved the continuous-time feasibility and formulation-space completeness of STCS. 
We experimentally validated the algorithm's performance in various scenarios with application to unsignalized traffic intersections, demonstrating that we expand upon the current state of the art with regard to dynamic obstacle compatibility and consistency in constrained settings, while maintaining runtime and solution quality.
As STCS is a novel approach to MAMP, there still remains much work to be done.
In future work, we intend to further explore optimization techniques and formally prove suboptimality bounds and adherence to motion constraints.

\end{document}